\documentclass[10pt,twocolumn,letterpaper]{article}

\usepackage{cvpr}              
\usepackage{microtype}      
\usepackage[table]{xcolor}
\definecolor{darkgrey}{rgb}{0.4,0.4,0.4} 
\usepackage{graphicx} 
\usepackage{amsthm}
\usepackage{amsmath, amssymb}
\usepackage{algorithm}
\usepackage{algpseudocode}
\usepackage{longtable}
\usepackage{booktabs}

\newtheorem{proposition}{Proposition}

\newcommand{\Kmat}[0]{\ensuremath{{\bf K}} }

\newcommand{\Qmat}[0]{\ensuremath{{\bf Q}} }

\newcommand{\Vmat}[0]{\ensuremath{{\bf V}} }
\newcommand{\Wmat}[0]{\ensuremath{{\bf W}} }

\newcommand{\bv}[0]{\ensuremath{\boldsymbol{b}} }
\newcommand{\cv}[0]{\ensuremath{\boldsymbol{c}} }

\newcommand{\pv}[0]{\ensuremath{\boldsymbol{p}} }

\newcommand{\xv}[0]{\ensuremath{\boldsymbol{x}} }

\newcommand{\zv}[0]{\ensuremath{\boldsymbol{z}} }

\newcommand{\muv}[0]{\ensuremath{\boldsymbol{\mu}} }

\usepackage{enumitem}
\usepackage{multirow} 
\usepackage{caption}
\usepackage{tabularx}
\usepackage[table]{xcolor}
\usepackage{makecell}  

\definecolor{ccr}{RGB}{10,110,150}  
\definecolor{cvprblue}{rgb}{0.21,0.49,0.74}
\usepackage[pagebackref,breaklinks,colorlinks,allcolors=cvprblue]{hyperref}

\def\paperID{3562} 
\def\confName{CVPR}
\def\confYear{2026}

\title{Improving Sparse Autoencoder with Dynamic Attention}

\author{Dongsheng Wang, Jinsen Zhang, Dawei Su, Hui Huang\thanks{Corresponding author.} \\
College of Computer Science and Software Engineering, Shenzhen University, China\\
\texttt{\{dongshengwang,2400101100,2023110003\}@szu.edu.cn},
\texttt{hhzhiyan@gmail.com}\\
}

\begin{document}
\maketitle
\begin{abstract}
Recently, sparse autoencoders (SAEs) have emerged as a promising technique for interpreting activations in foundation models by disentangling features into a sparse set of concepts. However, identifying the optimal level of sparsity for each neuron remains challenging in practice: excessive sparsity might lead to poor reconstruction, whereas insufficient sparsity harms interpretability. While existing activation functions such as ReLU and TopK provide certain sparsity guarantees, they typically require additional sparsity regularization or cherry-picked hyperparameters. We show in this paper that adaptive sparse attention mechanisms using sparsemax can bridge this trade-off, due to their ability to determine the number of concepts in a data-dependent manner. 
Specifically, we first explore a new class of SAEs based on the cross-attention architecture with the latent features as queries and the learnable dictionary as the key and value matrices. To encourage sparse pattern learning, we employ a sparsemax-based attention strategy that automatically infers a sparse set of concepts according to the complexity of each neuron, resulting in a more flexible and efficient activation function. Through comprehensive evaluation and visualization, we show that our approach successfully achieves lower reconstruction loss while producing high-quality concepts. Moreover, the sparsity level automatically determined by our approach can serve as tuning guidance to improve existing SAEs. The code is available \href{https://github.com/qyj-bkjx/Sparsemax-SAE}{https://github.com/qyj-bkjx/Sparsemax-SAE}.
\end{abstract}    
\vspace{-5pt}
\section{Introduction}
\label{sec:intro}
The impressive gains in reasoning and accuracy of recent large-scale machine learning models like CLIP~\cite{radford2021learning, cherti2023reproducible} and GPT~\cite{radford2019language,achiam2023gpt} have generally come at the cost of a loss of transparency into their functioning. Typically, neurons in these models are polysemantic, and they respond to seemingly unrelated inputs simultaneously. This can be explained as superposition~\cite{elhage2022superposition}, where models learn more independent features than they have neurons by viewing each feature as a linear combination of neurons. Fortunately, sparse autoencoders (SAEs)~\cite{olshausen1997sparse,lee2006efficient,robert_sae,hindupur2025projecting,fel2025archetypal} have emerged as a promising technique for addressing this fundamental challenge by learning an overcomplete yet sparse representation of neural activations, effectively disentangling these superimposed features into more interpretable concepts~\cite{templeton2024scaling,gao2024scaling,shi2025route, wang2022representing}.

\begin{figure}[!t]
\centering
\includegraphics[width=0.99\linewidth]{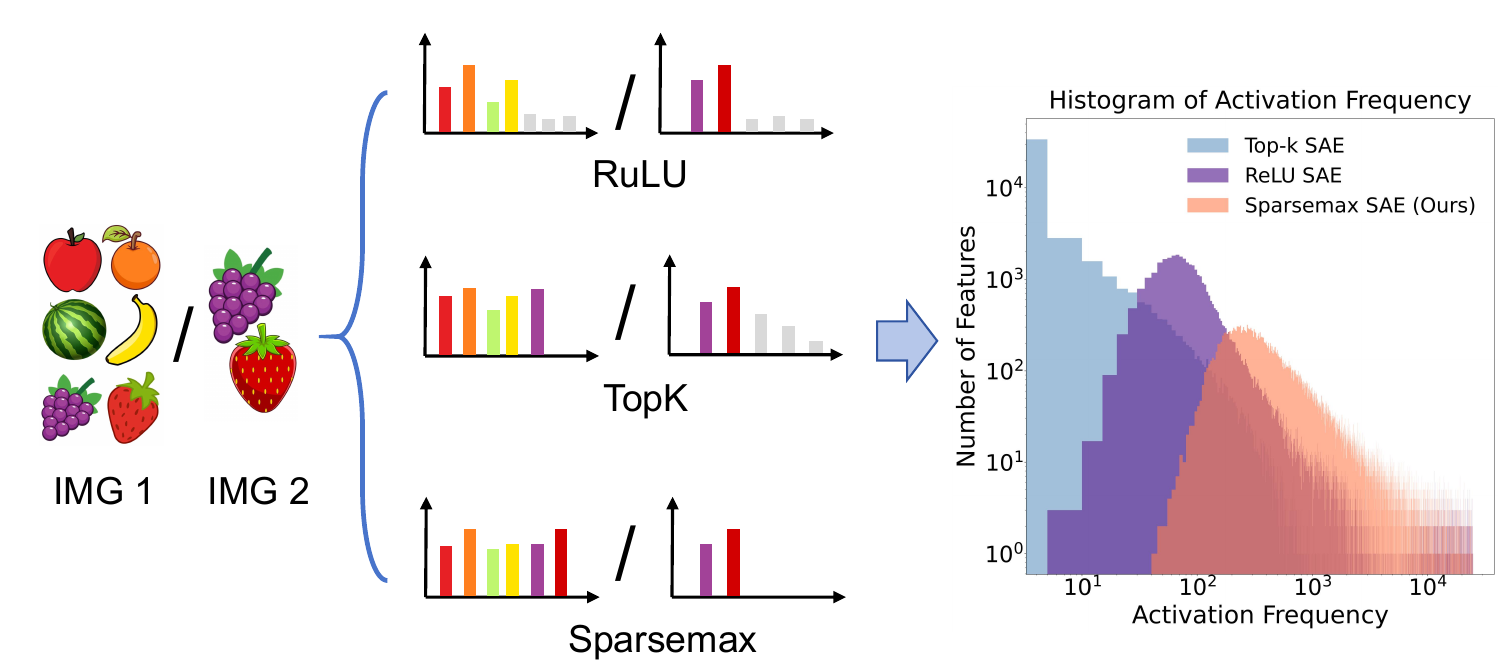}
\caption{\small{Comparisons of our proposed Sparsemax SAE with previous SAEs (left) and histogram of activation frequency of the learned concepts on the ImageNet dataset (right). ReLU-based SAEs often suffer from the feature shrinkage issue, while TopK-based models tend to produce dead concepts as they only keep $K$ largest concepts. In contrast, our Sparsemax-based SAE dynamically select the number of concepts based on feature complexity, thereby discovering more concepts.}}
\label{motivation_pic1}
\vspace{-2mm}
\end{figure}

Despite its successes in reversing the effects of superposition, determining the optimal sparse level for the latent features remains an open problem. For example, assigning too many concepts to each feature may compromise interpretability, while insufficient concepts can degrade reconstruction—both scenarios lead to suboptimal concept learning. Early SAEs~\cite{robert_sae,patchsae} adopt the ReLU as the activation function and combine it with $\mathcal{L}_1$ regularization to balance sparsity and reconstruction. However, the $\mathcal{L}_1$ penalty often leads to feature shrinkage, where all activations tend toward zero (Seen as Fig.~\ref{motivation_pic1}). GatedReLU~\cite{gatedrelu} suggests decomposing ReLU into direction selection and magnitude estimation functions via the gate mechanism. JumpReLU~\cite{rajamanoharan2024jumping} finds that zeroing out activations below a positive threshold is a better option. However, these models require additional regularizations to prompt the sparsity, and the balancing coefficient needs to be carefully selected to achieve satisfactory performance~\cite{hindupur2025projecting}.

 On the other hand, recent attempts propose to limit the number of concepts explicitly. For example, TopK SAEs~\cite{gao2024scaling} employ K-sparse autoencoder~\cite{makhzani2013k} to directly choose the $K$ largest concepts and zero the rest. This approach eliminates the need for an explicit sparsity penalty but imposes a rigid constraint on the number of active concepts per sample. BatchTopK SAEs~\cite{bussmann2024batchtopk} relax the top-K constraint to a batch-level constraint, enabling the SAEs to represent each sample within a batch with a variable number of concepts, resulting in more flexible and efficient utilization of the concept dictionary. Unfortunately, both TopK and BatchTopK view $K$ as a hyperparameter, and how to set $K$ properly remains an unsolved problem.

In this paper, we aim to improve SAEs with adaptive sparse attention mechanisms under the cross-attention framework. Moving beyond the traditional single-layer MLP-style encode-decode structure, we here explore a new class of SAEs based on the transformer architecture due to its successes in various tasks~\cite{vaswani2017attention,radford2018improving,peebles2023scalable,wang2024instruction}. Specifically, we first view the to-be-learned dictionary as a set of concept vectors, which will be used as the key and value matrices via the corresponding projections. For each latent feature in the neural networks, we view it as the query and apply the cross-attention operation to obtain the reconstructed feature. Intuitively, the calculation of attention weights can be viewed as the encoding stage of SAEs, whose output measures the relevance score between the query and concepts. Notably, this transformer-based SAE connects the encoding and decoding stages by sharing the same concept vectors, rather than viewing them as two independent MLPs, achieving coherent, high-quality concept learning.

With the designed transformer-based architecture, it is natural to replace the softmax with any well-studied sparse attention operations~\cite{attention1,attention2,goncalves2025adasplash,duan2021sawtooth}. Here, we adopt sparsemax~\cite{martins2016softmax} as our solution. On one hand, sparsemax is differentiable everywhere, and can be easily applied with gradient-based optimization. On the other hand, sparsemax has the ability to assign exactly zero probability to some of its outputs, sharing the same motivation as SAEs. Most importantly, unlike previous works that require hyperparameter $K$ to constrain the number of activations, sparsemax dynamically estimates a threshold function for each sample according to its complexity. This enables the model to output the most relevant concepts while truncating others to zero. As shown in Fig~\ref{motivation_pic1}, TopK-based SAEs sometimes fail to correctly assign the number of concepts due to the incorrect $K$. In contrast, our sparsemax successfully assigns six concepts to complex images and two concepts to simple images.
Our sparsemax can be viewed as a more precise version of BatchTopK, where we set $K$ at the sample level rather than the batch level, thereby achieving greater flexibility and accuracy in sparse estimation.

To sum up, our contributions are as follows:
\begin{itemize}
    \item We formalize a novel transformer-based SAEs under the cross-attention framework, which bridges the gap between the encoder and decoder of SAEs by sharing the same concept vectors, resulting in more coherent concept learning.
    \item A novel sparsemax function is developed to replace the original softmax function in the cross-attention operation. Sparsemax can determine the number of activations dynamically for each sample, without any regularization or hard TopK truncation.
    \item We provide extensive validation across image and text tasks, demonstrating that our approach not only captures coherent concepts but also achieves superior reconstruction results.
\end{itemize}



\section{Related Work}
\label{sec:related_work}

\paragraph{Mechanistic interpretability with SAEs.} 
Mechanistic interpretability aims to uncover and explain the black box characteristics, enabling models to understand input data and generate reasonable responses~\cite{rai2024practical}. Recently, Sparse Autoencoder (SAEs) have been applied to language models due to their inherent ability to generate interpretable latent concepts~\cite{sharkey2023taking,robert_sae}. Building upon the standard architecture with ReLU activation~\cite{bricken2023monosemanticity}, a series of studies have developed numerous improvements to the original design. For example, Gated SAE~\cite{gatedrelu}, Switch SAE~\cite{mudide273233368efficient} aim to design complex encoders to ensure sparse outputs; Focusing on the feature shrinkage issue~\cite{bricken2023monosemanticity}, JumpReLU~\cite{rajamanoharan2024jumping} introduces threshold parameters for each concept to truncate the concepts with small scores. Topk-based SAEs~\cite{bussmann2024batchtopk, gao2024scaling} directly keep $K$ concepts with large scores and zero out the others; In terms of the sparsity regularization, P-annealing~\cite{karvonen2408measuring} and mutual feature regularization (MFR)~\cite{marks2024enhancing} are designed as alternatives to $L_0$ and $L_1$ loss. 

Inspired by the successful application of SAE in LLMs, PatchSAE and its variants explore to train SAEs on top of the CLIP and DiNOv2~\cite{qiang2023interpretability,patchsae,stevens2025sparse,dinov2}, showing great potential in interpreting visual concepts. Additionally, there has also been interest in steering the generation process of diffusion models via SAEs~\cite{cywinski2025saeuron,gao2024scaling}. More recently, several studies have increasingly applied SAEs to multimodal LLMs~\cite{zhang2025large,pach2025sparse}, and show that SAE can learn shared concepts across the vision and text modalities. However these model primarily extend existing SAEs (such as ReLU and TopK) to the vision domain, without designing new SAE architectures.

\paragraph{Sparse Attention Mechanisms.}
Traditional attention mechanisms commonly employ the softmax transformation to convert scores into probability distributions~\cite{vaswani2017attention}. 
However, softmax produces dense distributions, assigning non-zero attention weights to all elements, which limits interpretability and efficiency~\cite{li2023patchct}.
SlidingWindow is a commonly used strategy that allows the query to compute attention only within a fixed windlow~\cite{xiao2023efficient,child2019generating,beltagy2020longformer}. However, these models rely on pre-defined sparse patterns, limiting their potential for application in SAEs. Top-K attention~\cite{gupta2021memory} shares similar ideas with TopK SAEs and selects a set of keys with the $K$ largest scores. SeerAttention~\cite{gao2024seerattention} separates queries and keys into spatial blocks and performs blockwise selection for sparsity. $\alpha$-entmax attention~\cite{peters2019sparse} provides natural, input-dependent sparsity patterns with an exact and differentiable transformation, attracting increasing attention in recent research. For example, Adaptively sparse transformers~\cite{correia2019adaptively} employ $\alpha$-entmax attention where attention heads can learn $\alpha$ dynamically. SparseFinder~\cite{treviso2022predicting} aims to address the efficiency issues of $\alpha$-entmax by predicting a prior. Sparse Flash Attention~\cite{gonccalvesadasplash,dao2022flashattention} combines the efficiency of GPU-optimized algorithms with the sparsity benefits of $\alpha$-entmax, showing great runtime and memory efficiency. Sparsemax~\cite{sparsemax} can be viewed as a special case of $\alpha$-entmax with $\alpha=2$. It maps inputs onto the probability simplex while allowing exact zeros in the output. Moreover, sparsemax yields piecewise-linear activations and a well-defined Jacobian, enabling efficient gradient computation. It has been successfully applied in multi-label classification and attention-based networks, producing more selective and interpretable attention maps without sacrificing differentiability~\cite{sparsemax,jin2025text,mudarisovlimitations}.

In this paper, we aim to introduce a new transformer-based SAEs based on the sparsemax attention. This approach enables SAE to be trained solely with the reconstruction loss, eliminating the need for additional penalty regularization or hyperparameter tuning. The proposed SAEs can also be easily applied to both visual and textual domains.

\section{Methodology}
In this section, we first briefly review SAEs, then introduce our proposed model in detail.

\begin{figure*}[!h]
\centering
\includegraphics[width=0.95\linewidth]{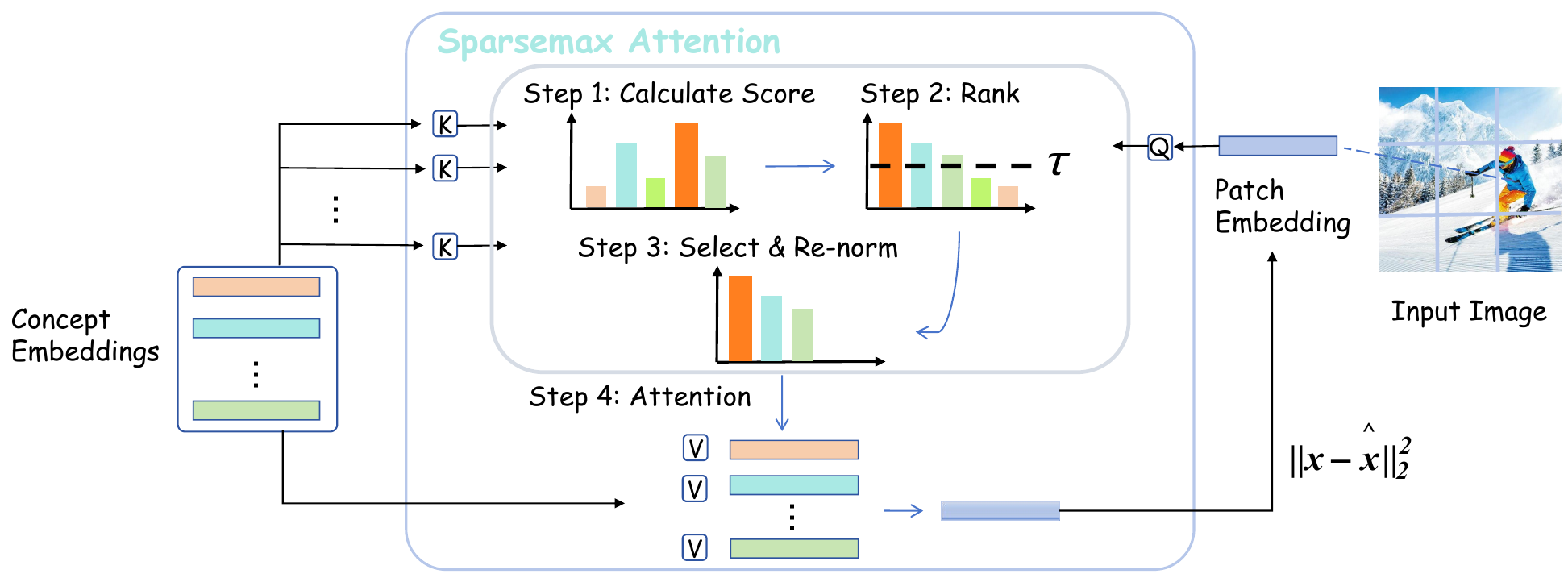}
\caption{\small{Framework of our Saprsemax SAE, which reconstructs the input feature under the transformer architectures, and the sparsemax attention is employed to dynamically determine the number of concepts by estimating the threshold $\tau$. }}
\label{framework}
\end{figure*}

\subsection{SAE and its variants}
To disentangle the polysemantic activations (or features) $\xv \in \mathbb{R}^d$ into a set of monosemantic and interpretable concepts $\mathcal{C} = \{ \cv_1, \cv_2,...,\cv_M \} \in \mathbb{R}^{d \times M}$, where $M \gg d$ represents the concept space dimension, SAEs typically represent $\xv$ as a sparse linear combination of these concepts under the encoder-decoder framework:
\begin{equation}\label{sae}
\begin{aligned}
    \zv &= \sigma (\Wmat_{\text{enc}} (\xv - \bv_{\text{enc}})),\\
    \hat{\xv} &= \Wmat_{\text{dec}} \zv + \bv_{\text{dec}},
\end{aligned}
\end{equation}
where $\Wmat_{\text{enc}} \in \mathbb{R}^{M \times d}$ and $\Wmat_{\text{dec}} \in \mathbb{R}^{d \times M}$ denote the weight matrices of the single-layer encoder and decoder, respectively. $\bv_{\text{enc}}, \bv_{\text{dec}} \in \mathbb{R}^d$ are two bias terms. The columns of $\Wmat_{\text{dec}}$ are the to-be-learned concepts $\mathcal{C}$, and the reconstruction $\hat{x}$ is obtained by weighting these concepts via $\zv$. To prompt the sparse combination, various structures of $\sigma$ have been developed in previous studies:

\noindent\textbf{ReLU-based $\sigma$}. Early SAEs often employ the ReLU activation function to generate sparse weights due to its simplicity of implementation. To address the feature shrinkage issue in ReLU, where activations in $\zv$ tend toward zero, GatedReLU and JumpRelu are proposed. Although effective in learning sparse $\zv$, these models require additional sparse regularizations, for example:
\begin{equation}
    \mathcal{L} = ||\xv - \hat{x}||^2_2 + \lambda S(\zv),
\end{equation}
where $S$ denotes the function penalizing non-sparse decompositions, \textit{e.g.}, $\mathcal{L}_1$ in ReLU and GatedReLU and $\mathcal{L}_0$ in JumpReLU. $\lambda$ sets the trade-off between sparsity and reconstruction, requiring careful tuning to achieve a balance between the two.

\noindent\textbf{TopK-based $\sigma$}. ~\citet{gao2024scaling} suggests that the TopK is another option to learn sparse $\zv$, where only the $K$ largest concepts are kept for each sample, with all others set to zero. Due to its explicit sparsity selection, TopK SAEs can be trained via the reconstruction loss. Built upon TopK, BatchTopK is further developed to replace the sample-level TopK operation with a batch-level BatchTopK function, where the top $n \times K$ activations across the entire batch of $n$ samples are selected, while all others are set to zero. Compared to ReLU-based SAEs, TopK-based SAEs empirically show a better balance between the sparsity and reconstruction. However, how to choose the optimal $K$ in these models remains an open problem.

In this paper, we aim to further relax the constraints of BatchTopK based on the sparsemax function, which dynamically determines the number of activations according to the feature's complexity, rather than setting a fixed $K$.

\subsection{Sparsemax SAE}
As illustrated in Fig.~\ref{framework}, we improve SAEs in two aspects: 1) a transformer-based architecture is designed to connect the encoder and decoder by sharing the same concept vectors; and 2) within the transformer-based framework, each sample has the ability to determine its sparse level dynamically via sparsemax attention. Below, let us introduce each of them in detail.

\paragraph{Transformer-based SAE.} Inspired by the great successes of transformer-based structures in various fields~\cite{vaswani2017attention,radford2018improving,peebles2023scalable,wang2023tuning}, we aim to explore a novel SAE with the cross-attention mechanism. Mathematically , we rewrite Eq.~\ref{sae} as:
\begin{equation} \label{cross_att}
\begin{aligned}
    \Qmat = \xv ^T \Wmat_Q , & \quad \Kmat = \mathcal{C} ^T \Wmat_K , \quad \Vmat = \mathcal{C}^T \Wmat_V,  \\ 
    \hat{x} &= \sigma(\frac{\Qmat \Kmat^T}{\sqrt{d}}) \Vmat,
\end{aligned}
\end{equation}
where $\Wmat_Q, \Wmat_K, \Wmat_V \in \mathbb{R}^{d \times d}$ denotes the query, key, and value projections, respectively. $\sigma$ is the sparsemax function, which will be introduced later. Intuitively, Eq.~\ref{cross_att} views the input feature as a query and reconstructs it using a set of concepts via the cross-attention framework. Compared to MLP-based SAEs in Eq.~\ref{sae} that directly output $\zv$ via $\Wmat_{enc}$, our approach models the activation weights as the similarity score of the query and concepts explicitly, thereby enabling more precise weight estimation. For example, a higher $\zv$ denotes a closer distance between the query and concepts in the embedding space. More importantly, unlike Eq.~\ref{sae} views $\Wmat_{\text{enc}}$ and $\Wmat_{\text{dec}}$ as two independent learnable projections, both the key and value in Eq.~\ref{cross_att} originate from the same concept $\mathcal{C}$. This reinforces the synergy between the concept vector $\Vmat$ and its weights during the weighting (decoding) stage, showing stronger reconstruction capabilities.

\paragraph{Sparsemax Attention.} One of the core ideas of SAEs is to assign a sparse set of concepts for each latent feature, while the widely-used softmax function in transformers often outputs dense activations~\cite{vaswani2017attention}. To this end, we introduce the sparsemax attention into our transformer-based SAE, as it is capable of producing exactly zero value to low-scoring concepts. Let $\zv = \Qmat \Kmat^T \in \mathbb{R}^M$ denote the similarity score between the query and $M$ concepts, sparsemax aims to project $\zv$ onto the probability simplex with Euclidean distance:
\begin{equation}\label{sparsemax4}
    \begin{aligned}
        \mathrm{sparsemax}(\zv) = \arg\min_{p \in \Delta^{M-1}} \| p - z \|^2,
    \end{aligned}
\end{equation}
where $\Delta^{M-1}:=\left\{\, p \in \mathbb{R}^M \;\middle|\; 
p_i \ge 0 \;\;, \sum_{i=1}^M p_i = 1 \right\}$ is the $(M-1)$-dimensional simplex. Sparsemax aims to find the point inside the simplex that is nearest to $\zv$. Therefore, for small coordinates of $\zv$, the closest point in the simplex will force them to zero, in which case $\mathrm{sparsemax}(\zv)$ becomes sparse. Fortunately, Eq.~\ref{sparsemax4} can be solved with linear-time algorithms~\cite{michelot1986finite,duchi2008efficient}.

\begin{proposition} \label{prop1}
The closed-form solution of Eq.~\ref{sparsemax4} is:
\begin{equation}
    \begin{aligned}
        \mathrm{sparsemax}(\zv)_m = \mathrm{max}(\zv_m-\tau, 0),
    \end{aligned}
\end{equation}
where $\tau$ is a threshold calculated so the result sums to 1: $\sum_{j \in S}(\zv_j - \tau)=1$ for every selected $\zv_j$, \textit{e.g.}, $S=\{j: \zv_j >\tau\}$. Furthermore, the support set $S$ (and hence $\tau$) can be computed efficiently via sorting: if we sort $z$ in descending order as $z_{(1)} \ge \cdots \ge z_{(M)}$, define
\begin{equation}
k = \max\left\{\, r \in \{1,\dots,M\} \,\middle|\, z_{(r)} + \frac{1 - \sum_{i=1}^{r} z_{(i)}}{r} > 0 \right\},
\end{equation}
then,
\begin{equation}
\tau = \frac{\sum_{i=1}^{k} z_{(i)} - 1}{k}.
\end{equation}
\end{proposition}
\begin{proof}
We provide a detailed derivation in the appendix.
\end{proof}

Unlike TopK-based algorithms that truncate $\zv$ by setting a hard threshold, Prop.~\ref{prop1} suggests a dynamical $\tau$ by measuring the content complexity of the input feature. For example, if the query feature consists of multiple concepts, the $\zv$ tends to contain many comparable values, resulting in a large set $S$. Conversely, if the query feature represents pure concepts, the resulting $S$ becomes very small. We summarize the Sparsemax algorithm in Alg.~\ref{sparsemax}.

\begin{algorithm} 
\caption{\textbf{Sparsemax Attention}}
\label{sparsemax}
\begin{algorithmic}[1]
\State \textbf{Input:} $\zv$
\State Sort $\zv$ as $z_{(1)} \ge \cdots \ge z_{(M)}$
\State Find $k(\zv) := \max\left\{\, r \in [M] \,\middle|\, z_{(r)} + \frac{1 - \sum_{i=1}^{r} z_{(i)}}{r} > 0 \right\}$, where $[M]:=\{1,\dots,M\}$
\State Define $\tau(\zv) = \frac{\sum_{i=1}^{k} z_{(i)} - 1}{k{\zv}}$.
\State \textbf{Output:} $\mathbf{p}$ such that $p_i = max\{0, z_i - \tau(\mathbf{z})\}$
\end{algorithmic}
\end{algorithm}

\begin{figure*}[!h]
    \centering
    \includegraphics[width=0.83\linewidth]{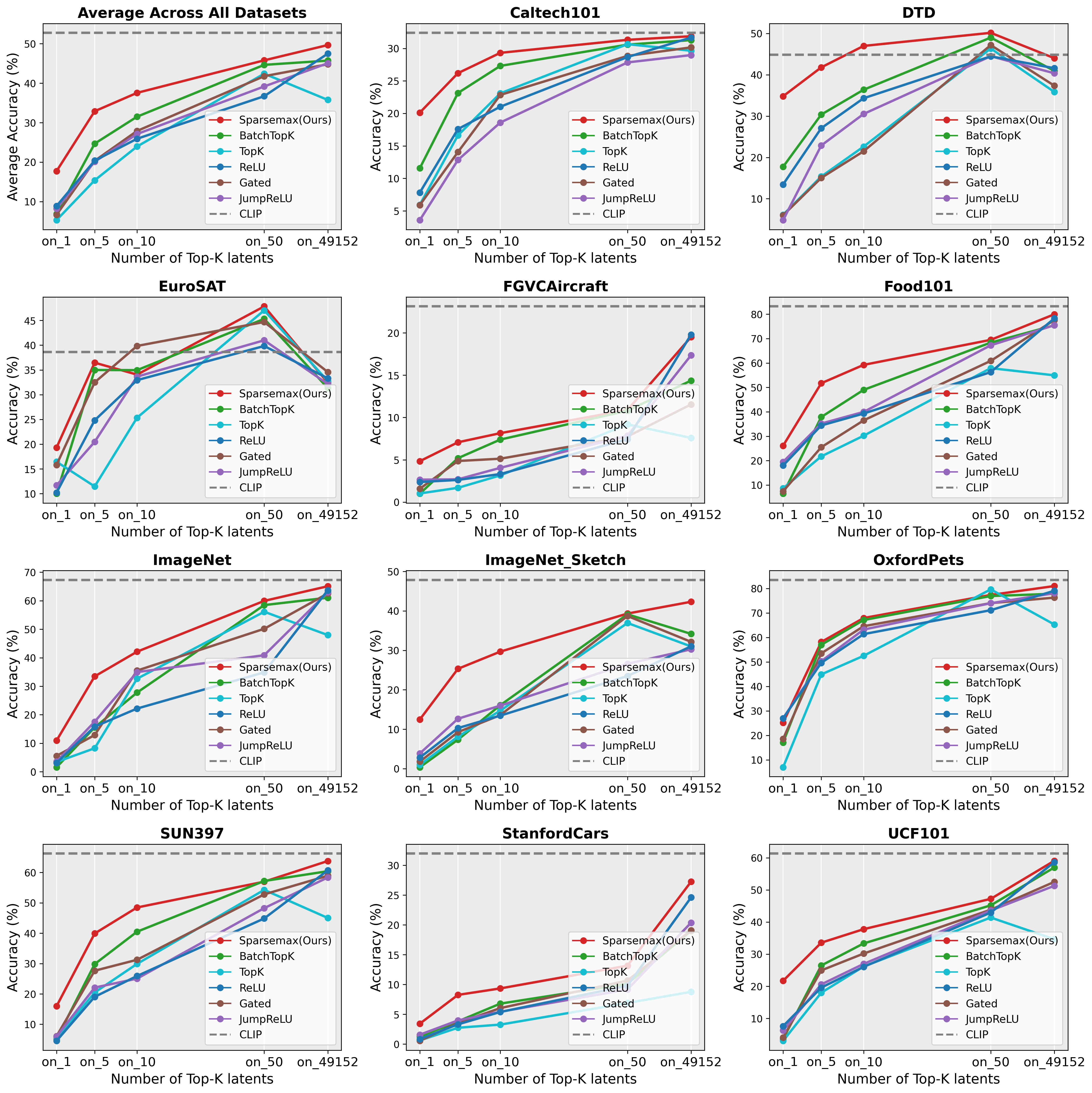}
    \caption{\small{Comparisions of zero-shot image classfication using top-n concepts on 11 datasets. All results are calculated as the mean value of three runs with different random seeds. CLIP denotes the results using the original image features in CLIP, and on\_49152 denotes all concepts are used.}}
    \label{fig:fsl}
    \vspace{-3mm}
\end{figure*}

\section{Experiments}
In this section, we first outline the experimental setup, followed by the evaluation of Sparsemax SAEs. We conduct experiments in both visual and textual domains and compare our approach with recent advances in terms of image classification and reconstruction tasks. Finally, we visualize the learned concepts at both the image and patch levels, revealing clear and interpretable visual patterns behind these concepts.


\begin{table*}[!ht]
  \centering
  \caption{\small{ NMSE scores with various dictionary sizes $M$ on the OpenWeb and WikiText-103 test datasets.}}
  \label{tab:nmse}
  \scalebox{0.9}{
  \begin{tabular}{lcccc|cccc}
    \toprule
    \multirow{2}{*}{\textbf{Method}} 
      & \multicolumn{4}{c}{\textbf{OpenWeb}} 
      & \multicolumn{4}{c}{\textbf{WikiText-103}} \\
    \cmidrule(lr){2-5} \cmidrule(lr){6-9}
      & $M{=}3072$ & $M{=}6144$ & $M{=}12288$ & $M{=}24576$
      & $M{=}3072$ & $M{=}6144$ & $M{=}12288$ & $M{=}24576$ \\
    \midrule
    RELU    & 0.064 & 0.064 & 0.064 &  0.059 & 0.064 & 0.064 & 0.064 &  0.064 \\
    JumpRELU    & 0.051 & 0.050 & 0.050 &  0.051 & 0.058 & 0.056 & 0.055 & 0.567 \\
    Gated       &  0.078  &  0.092  &  0.129  & 0.489  &0.088  & 0.106&  0.196 &0.527 \\
    TopK        & 0.014 & 0.059 & 0.010 & 0.055 & 0.024 & 0.063 & 0.018 & 0.061 \\
    BatchTopK   & 0.014 & 0.061 & 0.060 & 0.060 & 0.024 & 0.064 & 0.018 & 0.062 \\
    \textbf{Saprsemax SAE (Ours)} & \textbf{0.005} & \textbf{0.038} & \textbf{0.004} & \textbf{0.039}
                    & \textbf{0.008} & \textbf{0.046} & \textbf{0.007} & \textbf{0.045} \\
    \bottomrule
  \end{tabular}}
\end{table*}

\begin{table*}[!ht]
  \centering
  \caption{\small{CE degradation with various dictionary sizes $M$ on the OpenWeb and WikiText-103 test datasets.}}
  \label{tab:ce}
  \scalebox{0.9}{
  \begin{tabular}{lcccccccc}
    \toprule
    \multirow{2}{*}{\textbf{Method}} 
      & \multicolumn{4}{c}{\textbf{OpenWeb}} 
      & \multicolumn{4}{c}{\textbf{WikiText-103}} \\
    \cmidrule(lr){2-5} \cmidrule(lr){6-9}
      & $M{=}3072$ & $M{=}6144$ & $M{=}12288$ & $M{=}24576$
      & $M{=}3072$ & $M{=}6144$ & $M{=}12288$ & $M{=}24576$ \\
    \midrule
    Relu       &  -4.709   & -4.656 &  -4.614   & -1.562  &  -4.709   & -4.656 &  -4.614   & -4.615 \\
    JumpRELU        &  -0.586  &  -1.300  &  -0.928  &  -1.151  & -3.272  & -3.060  & -2.980  & -3.260  \\
    Gated       &  -1.738  &  -1.422 &   -2.679  & -1.309   &-6.637 & -5.791& -5.453 &-5.557  \\
    TopK            &  0.209  & -1.778  &  0.306   & -1.261  & -0.898 & -4.587  & -0.569 & -4.251 \\
    BatchTopK       &  0.196   & -1.742 &  0.302   & -1.740  & -0.867 & -4.659  & -0.566 & -4.356 \\
    \midrule
    \textbf{Saprsemax SAE (Ours)} & \textbf{0.031} & \textbf{-1.516} & \textbf{0.012} & \textbf{-0.395}
                    & \textbf{-0.106} & \textbf{-2.234} & \textbf{-0.113} & \textbf{-2.079} \\
    \bottomrule
  \end{tabular}}
\end{table*}

\begin{figure*}[t]
    \centering
    \includegraphics[width=0.9\textwidth]{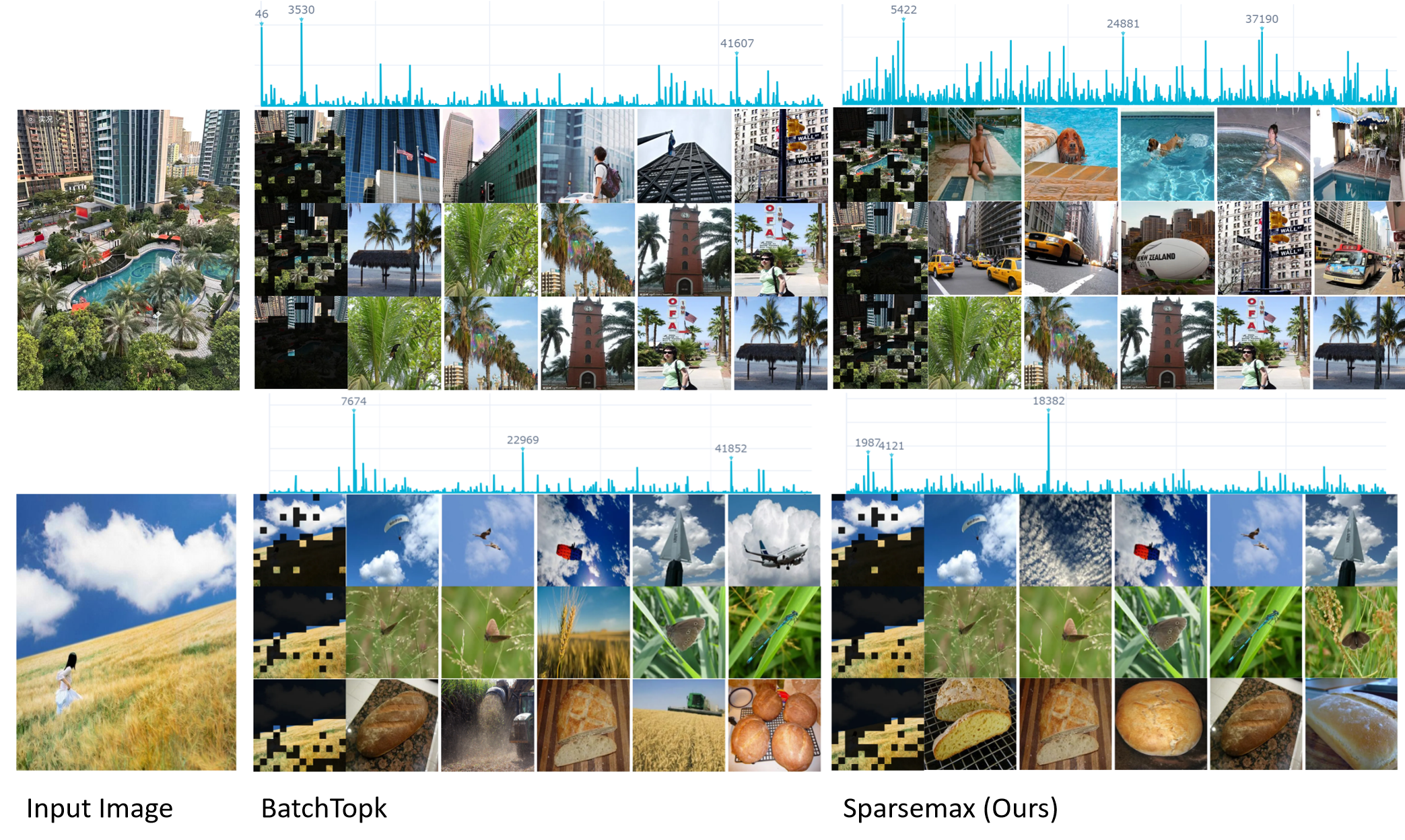}
    \caption{\small{Visualization of top three concepts given the reference image. For each concept, we provide its masking map within the input image and top five reference image from the ImageNet dataset. Compared to BatchSAE, our Sparsemax SAE learns clearer and more interpretable concepts.}}
    \label{fig:top3}
\end{figure*}

\subsection{Experimental Setup}
\paragraph{Datasets.} 
For the visual modality, we use the CLIP model with an image encoder of ViT-B/16. This results in a \texttt{CLS} and $14 \times 14$ image patch tokens as inputs. Following PatchSAE~\cite{patchsae}, we extract the ViT output from the residual stream of the second-last attention layer. Therefore, the training visual data has a size of $N_{img} \times N_{patch} \times d$, with $N_{img}$ and $N_{patch}$ denoting the number of images and patches per image, respectively. We train all SAEs on the ImageNet dataset~\cite{deng2009imagenet}, and evaluate the performance on 11 classification datasets: ImageNet~\citep{deng2009imagenet} and Caltech101~\citep{fei2004learning} for generic object classification, OxfordPets~\citep{parkhi2012cats}, StanfordCars~\citep{krause20133d}, Flowers102~\citep{nilsback2008automated}, Food101~\citep{bossard2014food} and FGVCAircraft~\citep{maji2013fine} for fine-grained image recognition, EuroSAT~\citep{helber2019eurosat} for satellite image classification, UCF101~\citep{soomro2012ucf101} for action classification, DTD~\citep{cimpoi2014describing} for texture classification, and SUN397~\citep{xiao2010sun} for scene recognition.
For the textual modality, we choose GPT-2 Small~\cite{radford2019language} as our pre-trained model, and extract the hidden embeddings from the residual stream of the $8$-th transformer layer. We train all SAEs on the training set of the OpenWebText dataset~\cite{Gokaslan2019OpenWeb}, which was processed into sequences of a maximum of 128 tokens for input into the language models. We report the reconstruction results on the test sets of the OpenWebText and WikiText-103 datasets. 
\vspace{-4mm}

\paragraph{Baselines}
We compare our Sparsemax SAE against a range of state-of-the art baselines, including \textbf{1) ReLU-based SAEs}: ReLUSAE~\cite{robert_sae}, a pioneering method that explains LLMs using SAE; 
PatchSAE~\cite{patchsae} views the image patches as tokens and extracts interpretable concepts at both the image and patch levels; 
GateSAE~\cite{gatedrelu} designs two encoders to model the position and coefficients of the concepts simultaneously;
JumpReLU~\cite{rajamanoharan2024jumping} introduces a learnable threshold into ReLU to alleviate the feature shrinkage issue. 
And \textbf{2) TopK-based SAEs}: 
TopKSAE~\cite{templeton2024scaling} directly keeps the $K$ largest concepts and zeroes out the others;
BatchTopK~\cite{bussmann2024batchtopk} relaxes TopKSAE by introducing the batch-level operation, where samples within a batch size share $K \times N_{batch}$ activations.

Unlike above models that require additional regularization loss or hyperparameter tuning, Our sparsemax SAE dynamically determines the optimal sparsity level based on feature complexity, demonstrating greater flexibility and interpretability.
\vspace{-8pt}
\paragraph{Implementation Details.} Following prior research~\cite{patchsae}, for the CLIP model, we set the number of concepts to $M=49,152$, as a value 64 times the latent dimension of the ViT model. The batch size is 32, and the training continued until a total of $2,621,440$ patches were feed. For the GPT-2 Small model, we conduct experiments with $M=3072,6144,12288$ and $24576$ to test the performance with different dictionary sizes. The batch size is 128 and training continued until a total of $1 \times 10^9$tokens were processed. All models were trained using the Adam optimizer with a learning rate of $3 \times 10^{-4}$, $\beta_1=0.9$ and $\beta_2=0.99$. For all baselines, we load the suggested hyperparameters according to their papers ($K=32$ for TopK-based SAEs and the sparsity weighting set to $1e^{-3}$ for ReLU-based SAEs). 

\subsection{Results Analysis}
\subsubsection{Zero-Shot Image Classification}
To measure the quality of the learned concepts of SAEs and explore whether the activated concepts per class capture the core class-level concepts, we conduct zero-shot image classification tasks by replacing the intermediate embeddings of ViT in CLIP with the SAE reconstruction embeddings using only the top $n=1,5,10,50$ concepts. 
The final prediction is calculated by the cosine similarity between textual prototypes and steered image features. To specify the subset concepts per class, we follow PatchSAE~\cite{patchsae} and first collect SAE latent activations from the training set of each dataset, and then select the largest $n$ concepts according to their activation frequency. These concepts are then act as masks to control the ViT's reconstruction. Intuitively, the selected top-n concepts capture key information of that class, and higher classification performance denotes a higher quality of the learned concepts.

Fig.~\ref{fig:fsl} reports the classification results of our Sparsemax SAE and baselines on 11 datasets. We also report the results with original ViT embeddings (CLIP) and results with all concepts (on\_49152). From the results, we have the following interesting finds: \textbf{1)} Overall, our Sparsemax SAE achieves the best average performance across 11 datasets on all top-n settings (top-left subfigure). Particularly at extremely small n values (i.e., $n=1,5,10$), our method significantly outperformed the second-best model. these results demonstrate that our sparse attention-based SAE is much more effective than ReLU and TopK-based SAEs. Our approach successfully assigns the most relevant concepts to the latent features, improving the representation learning of the dictionary. \textbf{2)} SAE-based models outperform CLIP on the EuroSAT and DTD datasets $n=10,50$ settings, whose images differ significantly from the pre-trained natural images. This demonstrates that the learned concepts are also generalizable and can be useful in zero-shot scenarios due to their monotonicity. \textbf{3)} Intuitively, more concepts should yield better reconstruction quality and thus higher prediction scores. However,we observe performance declining from on\_50 to on\_49152 across multiple datasets. we attribute this to the denoising capability of SAEs, where the top-n concepts capture the clear visual patterns that aid in identifying object labels. 

\subsubsection{Reconstruction Results on Text}
Following previous works~\cite{bussmann2024batchtopk}, we evaluate the reconstruction performance of our approach in terms of normalized mean squared error (NMSE) and cross-entropy (CE) degradation on the OpenWeb and WikiText-103 corpora, and report the results in Table.~\ref{tab:nmse} and ~\ref{tab:ce}, respectively. Similarly, we replace the intermediate features in GPT-2 Small with the reconstruction output of SAEs and measure the difference with the original outputs. Experiments demonstrate that across all datasets and dictionary sizes, our proposed Sparsemax SAE not only achieves significantly lower NMSE scores than other methods but also exhibits reduced CE degradation. This proves that Sparsemax SAE not only decouples polysemantic features into interpretable concepts but also reconstructs input data with lower information loss. The dynamic sparse attention mechanism enables the model to leverage more conceptual representations for complex features, fully demonstrating its effectiveness.

\begin{figure*}[!h]
    \centering
    \includegraphics[width=0.98\linewidth]{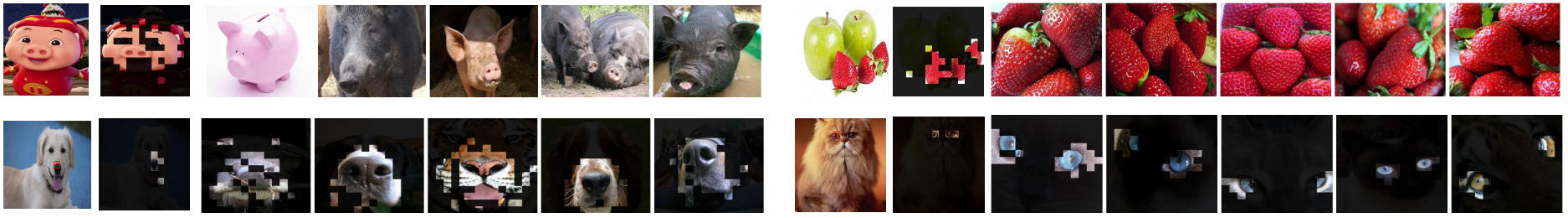}
    \caption{\small{Visualizations of the top one concept. For the query image, we interprete its most relevant concept from the image-level (top row) and patch-level (bottom row), respectively.}}
    \label{fig:top1}
    \vspace{-2mm}
\end{figure*}

\subsubsection{Ablation Study}
In this section, we aim to ablate the impacts of our proposed modules: transformer-based architecture and sparsemax attention. Specifically, we mainly consider two variants: MLP-based SAEs with the sparsemax function and transformer-based SAEs with the ReLU function. For the latter, we employ the $L_1$ regularization to promote sparsity. Table.~\ref{tab:imagenet} reports the image classification results on the ImageNet dataset. From the ablations, we find that both the introduced modules improve the performance of our baseline. The transformer framework connects the encoder and decoder by sharing the same concept vectors, resulting in more coherent concepts. The sparsemax function enables the model to dynamically determine the number of activations according to the feature's complexity, improving the trade-off between the sparsity and reconstruction.

\begin{table}[t]
  \centering
  \caption{\small{Ablation results on the ImageNet dataset.}}
  \vspace{-3mm}
  \label{tab:imagenet}
  \footnotesize
  \setlength{\tabcolsep}{4pt}
  \begin{tabular}{lrrrrr}
    \toprule
     & \texttt{on\_1} & \texttt{on\_5} & \texttt{on\_10} & \texttt{on\_50}& \texttt{on\_49152}  \\
    \midrule
    ReLU SAE    &  3.12 & 15.83 & 22.17 &34.87 & 63.67\\
    Transformer + ReLU    &3.86 &16.85  &24.08  & 36.33 & 63.94 \\
    MLP + Sparsemax     &7.91 &29.87  &39.73  &55.32  &64.74  \\
    \hline
    Sparsemax SAE (Ours)  & 10.93 & 33.47 & 42.13 &59.95 &65.09 \\
    \bottomrule
  \end{tabular}
\end{table}

\begin{table}[t]
  \centering
  \caption{\small{Zero-shot classification results of different $K$ on the Food101 dataset.}}
  \vspace{-3mm}
  \label{tab:food}
  \footnotesize
  \setlength{\tabcolsep}{4pt}
  \begin{tabular}{lrrrrr}
    \toprule
 & \texttt{on\_1} & \texttt{on\_5} & \texttt{on\_10} & \texttt{on\_50}& \texttt{on\_49152}  \\
    \midrule
    TopK ($K=32$)   &  8.64& 21.74 &30.21&57.89 & 54.96 \\
    BatchTopK ($K=32$) &6.46&37.86&48.99& 68.37 & 75.44 \\
    \hline
    TopK ($K=24$)   & 0.99& 24.42&42.88 &66.80 &47.67   \\
    BatchTopK ($K=24$)& 7.36&44.53 &49.52 &69.40 & 75.60  \\
    \hline
    Sparsemax SAE (Ours) & 26.11&51.71& 59.23 &69.49 &79.95 \\
    \bottomrule
  \end{tabular}
\end{table}

\subsubsection{Further Analysis}
Our sparsemax SAE is trained solely with reconstruction loss, with sparsity determined by the dataset. This enables optimal $K$ selection for TopK-based SAE.Specifically, we first collect all activations of our sparsemax SAE on the ImageNet dataset, and calculate the average number of activated concepts per sample $K^*=24$. We then re-train TopK-based SAE with the new $K^*$. Table.~\ref{tab:food} reports the results of $K=24,32$. We find that the calculated $K$ using our approach is a better choice for both TopK and BatchTopK SAEs in most cases. This improvement suggests that our sparsemax attention is able to estimate the level of sparsity.

\subsubsection{Visualization Results}
In addition to the above quantitative analysis, we also provide visualizations of the learned concepts. First, we aim to evaluate the top three concepts activated for each given image. Specifically, we use the \texttt{CLS} token as the global representation of the images, and obtain the top three concepts with the largest activation scores. Fig.~\ref{fig:top3} shows the comparison results of BatchTopK and our Sparsemax SAE. For each model, we also visualize the activation scores of all concepts. For each selected concept, we visualize its corresponding patches of the input image and the top five reference images from the ImageNet dataset, which provide an easy tool to understand the meaning of the concepts. From the results, we find that our approach successfully disentangles the core concepts from the input image\cite{liu2024patch},and the top three concepts show clear and specific visual patterns. For example, Sparsemax SAE successfully extracts concepts of the swimming pool, building, and tree from the input image. Compared to BatchTopK, we find that both models are able to extract the relevant concepts. However, BatchTopK sometimes produces redundant or unclear patterns. For example, the second and third concepts of the first image learned from BatchTopK share similar reference images, and the third concept of the second image contains both the wheat and bread patterns.

We also provide the visualization of the top one concept at the image level (top row of Fig.~\ref{fig:top1}) and patch level (bottom row of Fig.~\ref{fig:top1}). We find that the top one concept successfully extracts the key information of the given image. For the patch-level visualization, our approach accurately localizes the dog's nose and cat's eyes in the reference images, which demonstrates the high-quality of the learned concepts. This suggests that the learned concepts exhibit internal semantic structure that could be further modeled with graph-based discovery methods \cite{duan2021topicnet}.

\vspace{-3mm}
\section{Conclusion}
In this paper, we present a novel transformer-based SAE with sparsemax attention mechanism, where the input feature acts as the query and the concepts are modeled as the key and value matrices. The sparsemax function is then employed to steer the sparsity between the query and key attention. The transformer structure connects the encoder and decoder of SAEs by sharing the same concept vectors, while the sparsemax function enables the model to dynamically determine the sparsity level based on the feature's complexity. This synergy enhances the concept learning and reconstruction performance of traditional SAEs. Extensive experiments across image classification, text reconstruction, and visualization validate the effectiveness of our model. We hope our sparsemax SAE will offer novel insights for secure artificial intelligence and interpretability research.

\clearpage
\subsection*{Acknowledgments}

This work was supported in part by National Key R\&D Program of China (2024YFB3908500, 2024YFB3908502), NSFC (62506237, 62576215), ICFCRT (W2441020), Guangdong Basic and Applied Basic Research Foundation (2023B1515120026), Shenzhen Science and Technology Program (KQTD20210811090044003, KJZD2024 0903100022028, RCJC20200714114435012), and Scientific Development Funds from Shenzhen University.
{
    \small
    \bibliographystyle{ieeenat_fullname}
    \bibliography{main}
}
\clearpage
\setcounter{page}{1}
\maketitlesupplementary

\section{Proof of Prop.~\ref{prop1}}

\paragraph{1. Problem formulation.}
Recall that sparsemax is the Euclidean projection of $\zv\in\mathbb R^M$ onto the probability simplex
\begin{equation}
    \Delta^{M-1} = \{\pv\in\mathbb R^M \mid \mathbf{1}^\top\pv = 1,\ \pv\ge\mathbf{0}\},
\end{equation}
i.e.
\begin{equation}
    \mathrm{sparsemax}(\zv)=\arg\min_{\pv\in\Delta^{M-1}} \tfrac{1}{2}\|\pv-\zv\|^2_2.
\end{equation}

Equivalently, we solve the constrained quadratic program
\begin{equation}
    \min_{\pv\in\mathbb R^M}\ \frac{1}{2}\sum_{i=1}^M (p_i - z_i)^2
\quad\text{subject to}\quad
\sum_{i=1}^M p_i = 1,\quad p_i \ge 0\ \forall i.
\end{equation}

\paragraph{2. Lagrangian and KKT conditions.}
Introduce the Lagrange multiplier $\tau\in\mathbb R$ for the equality constraint $\sum_i p_i = 1$ and multipliers $\mu_i\ge0$ for the inequality constraints $p_i\ge0$. The (augmented) Lagrangian is
\begin{equation}
    \mathcal L(\pv,\tau,\muv)
= \frac{1}{2}\sum_{i=1}^M (p_i - z_i)^2
+ \tau\Big(\sum_{i=1}^M p_i - 1\Big)
- \sum_{i=1}^M \mu_i p_i.
\end{equation}
Karush--Kuhn--Tucker (KKT) conditions for optimality are: \\
(i) \emph{Primal feasibility}: $\sum_{i=1}^M p_i = 1$, $p_i \ge 0$ for all $i$.\\
(ii) \emph{Dual feasibility}: $\mu_i \ge 0$ for all $i$.\\
(iii) \emph{Stationarity}:
\begin{equation} \label{eq12}
    \frac{\partial \mathcal L}{\partial p_i} = p_i - z_i + \tau - \mu_i = 0
\quad\Longrightarrow\quad
p_i = z_i - \tau + \mu_i,\qquad \forall i.
\end{equation}

(iv) \emph{Complementary slackness}:
\begin{equation}\label{eq13}
    \mu_i p_i = 0,\qquad \forall i.
\end{equation}

Given Eq.~\ref{eq12} and Eq.~\ref{eq13}, we have:

\begin{itemize}
  \item If $p_i > 0$, complementary slackness forces $\mu_i = 0$. Hence
  \[
  p_i = z_i - \tau \qquad\text{and thus}\qquad z_i - \tau > 0.
  \]
  \item If $p_i = 0$, complementary slackness allows $\mu_i \ge 0$ and stationarity gives
  \[
  0 = z_i - \tau + \mu_i \quad\Longrightarrow\quad \mu_i = \tau - z_i.
  \]
  Since $\mu_i \ge 0$, we conclude $\tau - z_i \ge 0$, i.e. $z_i \le \tau$.
\end{itemize}

Combining both cases yields the compact expression (Eq.~\ref{sparsemax4} in the main paper)
\begin{equation}
    p_i = \max(z_i - \tau,\,0),\qquad \forall i,
\end{equation}

\paragraph{3. Determining $\tau$ and the active set.}
Let the active set (support) be
\begin{equation}
S = \{i \mid p_i > 0\} = \{i \mid z_i > \tau\},\qquad k := |S|.
\end{equation}

Summing $p_i = z_i - \tau$ over $i\in S$ and using $\sum_{i=1}^M p_i = 1$ gives
\begin{equation}
    \sum_{i\in S} (z_i - \tau) = 1
\quad\Longrightarrow\quad
\sum_{i\in S} z_i - k\tau = 1.
\end{equation}
Thus
\begin{equation}
\tau = \frac{\sum_{i\in S} z_i - 1}{k}.
\end{equation}

To find $S$ efficiently, sort the coordinates in descending order:
\begin{equation}
z_{(1)} \ge z_{(2)} \ge \cdots \ge z_{(M)}.
\end{equation}

If the optimal support corresponds to the top $k$ indices (this is always the case: if some $j\in S$ were not among the top $k$, there would be an index in the top $k$ not in $S$ with a larger $z$, contradicting $z_j>\tau$), then
\begin{equation}
\tau_k = \frac{\sum_{i=1}^k z_{(i)} - 1}{k}.
\end{equation}

The correct $k$ is the largest integer for which the $k$-th sorted element exceeds this threshold:
\begin{equation}
z_{(k)} > \tau_k
\quad\Longleftrightarrow\quad
z_{(k)} + \frac{1-\sum_{i=1}^k z_{(i)}}{k} > 0.
\end{equation}

Therefore set
\begin{equation}
    k = \max\Big\{ r\in\{1,\dots,M\}\ \Big|\ z_{(r)} + \frac{1 - \sum_{i=1}^{r} z_{(i)}}{r} > 0\Big\},
\end{equation}
and then take $\tau=\tau_k$, which completes the proof.


\section{More Visualization Results}
\paragraph{Analysis of Top 3 Concepts.}
Fig.~\ref{fig:fruit}-~\ref{fig:wlof} show the visualizations of the top three concepts of the test image. Overall, we find that our Sparsemax SAE successfully captures the key visual patterns. For example, the mixture of fruit, wooden background, and apples in Fig.~\ref{fig:fruit}. Moreover, our approach is able to understand visual features from different perspectives. For example, the first concept in Fig.~\ref{fig:pig} corresponds to a pig (the main object of the input image), the second concept is related to cartoon characters, and the third concept is about the dressing.

\begin{table}[!t]
  \centering
  \caption{Sparsity metric values.}
  \label{tab:cknna_subset}
  \begin{tabular}{lccccc}
    \toprule
    Model & $L_0\uparrow$ & FVU$\downarrow$ & CS$\uparrow$ & CKNNA$\uparrow$ & DO$\downarrow$ \\
    \midrule
    ReLU        & 0.928 & \textbf{0.098} & \textbf{0.953} & \textbf{0.812} & 0.003 \\
    TopK        & 0.966 & 0.169 & 0.925 & 0.701 & 0.003 \\
    BatchTopK   & 0.814 & 0.278 & 0.904 & 0.750 & 0.002 \\
    Ours        & \textbf{0.979} & 0.129 & 0.934 & 0.796 & \textbf{0.001} \\
    \bottomrule
  \end{tabular}
\end{table}

\begin{table}[!t]
  \centering
  \caption{Interpretability metrics.}
  \label{tab:interp_metrics}
  \begin{tabular}{lcc}
    \toprule
    Model & MEAN-MS & MAX-MS \\
    \midrule
    ReLU        & 0.1627 & 0.9172 \\
    TopK        & 0.0548 & 0.8751 \\
    BatchTopK   & 0.1243 & 0.9031 \\
    Ours        & \textbf{0.3484} & \textbf{0.9575} \\
    \bottomrule
  \end{tabular}
\end{table}

\begin{table*}[!h]
\centering
\caption{\small{Zero-shot classification results on the Caltech101 dataset. $K=32$ in TopK and BatchTopK.}}
\label{tab:calth101_transposed}
\scalebox{0.9}{
\begin{tabular}{ccccccccccccc}
\toprule
 & \texttt{no\_sae} & \texttt{on\_1} & \texttt{on\_5} & \texttt{on\_10} &  \texttt{on\_50} &\texttt{on\_49152} \\
\midrule
ReLU   & 32.403 &  7.816 & 17.591 & 21.033 & 28.671& 31.672  \\
JumpReLU   & 32.403 &3.578 &12.840  &18.594  &27.845  &28.974  \\
Gated    & 32.403 &5.894 &14.054  &22.847  &28.847  &30.158  \\
TOPK   & 32.403 &5.94 & 16.632 & 23.096 & 30.654 & 29.645 \\
BatchTopK   & 32.403 &  11.567 & 23.144 &27.314 & 30.582 &31.284   \\
\hline
Saprsemax SAE (Ours) & 32.403 & 20.121 & 26.195 & 29.322  &31.325 &31.875  \\
\bottomrule
\end{tabular}}
\end{table*}

\begin{table*}[!t]
\centering
\caption{\small{Zero-shot classification results on the DTD dataset. $K=32$ in TopK and BatchTopK.}}
\label{tab:dtd_transposed}
\scalebox{0.9}{
\begin{tabular}{ccccccccccccc}
\toprule
 & \texttt{no\_sae} & \texttt{on\_1} & \texttt{on\_5} & \texttt{on\_10} &  \texttt{on\_50} &\texttt{on\_49152}   \\
\midrule
ReLU   & 44.840 & 13.457 & 27.074 & 34.326  &44.468& 41.596  \\
JumpReLU   & 44.840 &4.825 &22.884  &30.520  &44.495  &40.367  \\
Gated    & 44.840 &5.989 &15.048  &21.495  &47.187  &37.365  \\
TopK   & 44.840 &  6.117 & 15.372 & 22.624  & 46.365&35.851  \\
BatchTopK  & 44.840 & 17.730 &30.372 & 36.383 &49.007& 40.957  \\
\hline
Saprsemax SAE (Ours) & 44.840 & 34.804 & 41.791 & 46.986 & 50.16 &44.025&  \\

\bottomrule
\end{tabular}}
\end{table*}

\begin{table*}[!t]
\centering
\caption{\small{Zero-shot classification results on the EuroSAT dataset. $K=32$ in TopK and BatchTopK.}}
\label{tab:eurosat_transposed}
\scalebox{0.9}{
\begin{tabular}{ccccccccccccc}
\toprule
 & \texttt{no\_sae} & \texttt{on\_1} & \texttt{on\_5} & \texttt{on\_10} & \texttt{on\_50} & \texttt{on\_49152} & \\
\midrule
ReLU   & 38.605 & 10.200 & 24.826 & 32.939& 39.832&  33.315 \\
JumpReLU   & 38.605 &11.680 &20.475  &33.648  &40.987  &32.158  \\
Gated    & 38.605 &15.815 &32.487  &39.846  &44.682  &34.577  \\
TopK   & 38.605 & 16.46 & 11.468 & 25.344& 47.041 & 32.741 \\
BatchTopK &38.605 & 10.0 & 34.990 & 34.957& 45.317& 31.338  \\
\hline
Saprsemax SAE (Ours) & 38.605 & 19.301& 36.463 & 34.056&47.839& 32.421 \\
\bottomrule
\end{tabular}}
\end{table*}

\begin{table*}[!t]
\centering
\caption{\small{Zero-shot classification results on the FGVC dataset. $K=32$ in TopK and BatchTopK.}}
\label{tab:fgvc_transposed}
\scalebox{0.9}{
\begin{tabular}{ccccccccccccc}
\toprule
 & \texttt{no\_sae} & \texttt{on\_1} & \texttt{on\_5} & \texttt{on\_10}&\texttt{on\_50} & \texttt{on\_49152}  \\
\midrule
ReLU   & 23.137 &  2.371 &  2.617 &  3.314 &7.431 &  19.803  \\
JumpReLU   & 23.137 &2.647 &2.689  &4.047  &7.846  &17.358  \\
Gated    & 23.137 &1.574 &4.855  &5.128  &7.748  &11.547  \\
TopK   & 23.137 &  1.02&  1.683  &  3.156&9.18  &7.59 \\
BatchTopK & 23.136 &  0.990 & 5.191 & 7.390 & 10.805& 14.362  \\

\hline
Saprsemax SAE (Ours) & 23.137 &  4.8346 &  7.069& 8.14871 & 10.894 &19.5187 \\
\bottomrule
\end{tabular}}
\end{table*}

\paragraph{Failure Cases.}
We also provide several failure cases of our Sparsemax SAE in Fig.~\ref{fig:messi}-~\ref{fig:football}. On one hand, we find that the learned concepts sometimes contain unclear visual patterns (The third concept in Fig.~\ref{fig:messi} and the first concept in Fig.~\ref{fig:football}). This may stem from feature absorption issues, where the learned concepts fail to decompose into into their subconcepts. On the other hand, the learned concepts occasionally share the similar masking concent of the input image. We attribute this to the fine-grained feature of the learned concepts, where concepts capture the similar visual patterns while focusing on distinct dimensions. For example, the concepts of baby, cute girl, and playing girl in Fig.~\ref{fig:girl}.

\paragraph{Sparsity and Interpretability Analysis}
To further evaluate the quality of the learned concepts, we report sparsity metric values in Table.~\ref{tab:cknna_subset} and interpretability metrics in Table.~\ref{tab:interp_metrics}. Following prior work~\cite{karvonen2408measuring, marks2024enhancing}, we compute metrics including $L_0$ (higher is better), FVU (fraction of variance unexplained, lower is better), CS (cosine similarity, higher is better), CKNNA (higher is better), and DO (dead concepts, lower is better). For interpretability, we follow the evaluation protocol in \cite{pach2025sparse} and report both the mean and maximum monosemantic scores of the learned concepts. Our Sparsemax SAE achieves the best performance under most metrics, demonstrating that the dynamic attention mechanism not only produces sparse representations but also yields more interpretable and semantically coherent concepts.

\section{Results of Zero-shot Image Classification}
We report the detailed zero-shot image classification results in Table.~\ref{tab:calth101_transposed}-~\ref{tab:ucf101_transposed}.

\begin{table*}[t]
\centering
\caption{\small{Zero-shot classification results on the Food101 dataset. $K=32$ in TopK and BatchTopK.}}
\label{tab:food101_transposed}
\scalebox{0.9}{
\begin{tabular}{ccccccccccccc}
\toprule
 & \texttt{no\_sae} & \texttt{on\_1} & \texttt{on\_5} & \texttt{on\_10} & \texttt{on\_50} & \texttt{on\_49152} \\
\midrule
ReLU   & 83.195 & 17.991 & 34.420 & 39.286 &56.2719 &78.285  \\
JumpReLU   & 83.195 &19.475 &35.187  &39.954  &67.257  &75.481  \\
Gated    & 83.195 &7.458 &25.486  &36.487  &60.875  &77.584  \\
TopK   & 83.195 & 8.644& 21.735 & 30.207 &  57.888&54.956  \\
BatchTopK & 83.194 & 6.459 & 37.864 & 48.996 & 68.365 &75.440\\

\hline
Saprsemax SAE (Ours) & 83.195 &26.1134&51.705& 59.23& 69.49& 79.954  \\
\bottomrule
\end{tabular}}
\end{table*}

\begin{table*}[t]
  \centering
  \caption{\small{Zero-shot classification results on the ImageNet dataset. $K=32$ in TopK and BatchTopK.}}
  \label{tab:imagenet}
  \scalebox{0.9}{
  \begin{tabular}{lrrrrrrrrrrrr}
    \toprule
    & \texttt{no\_sae} & \texttt{on\_1} & \texttt{on\_5} & \texttt{on\_10} & \texttt{on\_50}& \texttt{on\_49152}  \\
    \midrule
    ReLU   & 67.294 &  3.116 & 15.829 & 22.173&34.872 & 63.670\\
    JumpReLU   & 67.294 &3.548 &17.558  &34.975  & 40.876 & 62.957 \\
    Gated    & 67.294 &5.568 &12.875  &35.495  &50.159  &62.547  \\
    TopK   & 67.294 &  3.421 & 8.2501 & 32.646&56.135 & 47.956  \\
    BatchTopK & 67.294 &1.58 &15.562 & 27.782& 58.472 & 61.072 \\
    \hline
    Saprsemax SAE (Ours) & 67.294 & 10.929 & 33.469 & 42.127 &59.947 &65.087 \\
    \bottomrule
  \end{tabular}}
\end{table*}

\begin{table*}[t]
  \centering
  \caption{\small{Zero-shot classification results on the ImageNet-Sketch dataset. $K=32$ in TopK and BatchTopK.}}
  \label{tab:imagenet-sketch}
  \scalebox{0.9}{
  \begin{tabular}{lrrrrrrrrrrrr}
    \toprule
    & \texttt{no\_sae} & \texttt{on\_1} & \texttt{on\_5} & \texttt{on\_10}& \texttt{on\_50} &  \texttt{on\_49152}  \\
    \midrule
    ReLU   & 47.851 &  2.798 & 10.283 & 13.451&23.479& 31.009  \\
    JumpReLU   & 47.851 &3.847 &12.657  &15.984  &26.581  &30.257  \\
    Gated    & 47.851 &1.782 &9.257  &13.576  &38.712  &32.157  \\
    TopK   & 47.851 &  0.957 &  8.1269 & 14.712&36.942& 30.991  \\
    BatchTopK & 47.851 &  0.3509& 7.349 &16.103&39.128& 34.178  \\
    \hline
    Saprsemax SAE (Ours) & 47.851 & 12.460 & 25.321 & 29.678& 39.328& 42.314 \\
    \bottomrule
  \end{tabular}}
\end{table*}

\begin{table*}[t]
  \centering
  \caption{\small{Zero-shot classification results on the OxfordPets dataset. $K=32$ in TopK and BatchTopK.}}
  \label{tab:oxfordpets}
  \scalebox{0.9}{
  \begin{tabular}{lrrrrrrrrrrrr}
    \toprule
    & \texttt{no\_sae} & \texttt{on\_1} & \texttt{on\_5} & \texttt{on\_10} & \texttt{on\_50} & \texttt{on\_49152}  \\
    \midrule
    ReLU   & 83.477 & 26.964 & 49.611 & 61.350& 71.127& 79.061  \\
    JumpReLU   & 83.477 &26.981 &50.258  &63.204  &73.980  &78.041  \\
    Gated    & 83.477 &18.547 &53.492  &64.581  &74.012  &76.251  \\
    TopK   & 83.477 &  6.973 & 44.886 & 52.492& 79.621& 65.274  \\
    BatchTopK & 83.477 & 17.081 & 56.985 & 67.100 &76.918 & 77.828 \\
    \hline
    Saprsemax SAE (Ours) & 83.477 & 25.186 &58.135 &67.925 &77.351 &80.986\\
    \bottomrule
  \end{tabular}}
\end{table*}

\begin{table*}[t]
  \centering
  \caption{\small{Zero-shot classification results on the StandfordCars dataset. $K=32$ in TopK and BatchTopK.}}
  \label{tab:standfordcars}
  \scalebox{0.9}{
  \begin{tabular}{lrrrrrrrrrrrr}
    \toprule
    & \texttt{no\_sae} & \texttt{on\_1} & \texttt{on\_5} & \texttt{on\_10}& \texttt{on\_50} & \texttt{on\_49152}   \\
    \midrule
    ReLU   & 31.976 & 0.828 & 3.307 & 5.383 &9.778 & 24.602  \\
    JumpReLU   & 31.976 &1.568 &3.964  &5.421  &9.157  &20.367  \\
    Gated    & 31.976 &0.540 &3.541  &6.034  &10.652  &19.068  \\
    TopK   & 31.976 & 0.610 & 2.735 & 3.263  &6.897 &8.7586  \\
    BatchTopK & 31.976 & 1.248 & 3.811 & 6.786 &10.1140 & 18.987\\
    \hline
    Saprsemax SAE (Ours) & 31.976 &3.421 & 8.25 &9.337&13.155  & 27.245 \\
    \bottomrule
  \end{tabular}}
\end{table*}

\begin{table*}[t]
  \centering
  \caption{\small{Zero-shot classification results on the SUN397 dataset. $K=32$ in TopK and BatchTopK.}}
  \label{tab:sun397}
  \scalebox{0.9}{
  \begin{tabular}{lrrrrrrrrrrrr}
    \toprule
    & \texttt{no\_sae} & \texttt{on\_1} & \texttt{on\_5} & \texttt{on\_10} & \texttt{on\_50}  & \texttt{on\_49152} \\
    \midrule
    ReLU   & 66.257 &  4.543 & 19.045 & 25.902& 44.859 & 60.704\\
    JumpReLU   & 66.257 &5.962 &22.084  &25.035  &48.258  &58.367  \\
    Gated    & 66.257 &6.058 &27.643  &31.247  &52.796  &58.947  \\
    TopK   & 66.257 &  5.2819 &20.468 &29.9 &54.277 & 45.027 \\
    BatchTopK & 66.257 & 5.117 &29.824 & 40.5  & 57.203& 60.431 \\
    \hline
    Saprsemax SAE (Ours) & 66.257 & 15.946 & 39.904& 48.4794  &56.973&63.788 \\
    \bottomrule
  \end{tabular}}
\end{table*}

\begin{table*}[t]
  \centering
  \caption{\small{Zero-shot classification results on the UCF101 dataset. $K=32$ in TopK and BatchTopK.}}
  \label{tab:ucf101_transposed}
  \scalebox{0.9}{
  \begin{tabular}{lrrrrrrrrrrrr}
    \toprule
    & \texttt{no\_sae} & \texttt{on\_1} & \texttt{on\_5} & \texttt{on\_10}& \texttt{on\_50} & \texttt{on\_49152}  \\
    \midrule
    ReLU   & 61.322 &  7.607 & 19.600 & 26.056&42.945 & 58.626  \\
    TopK   & 61.322 &  3.039 & 18.003& 26.157 & 41.396& 34.647 \\
    BatchTopK & 61.322 & 3.591 & 26.474 & 33.346&45.194 &  56.986  \\
    JumpReLU   & 61.322 &6.257 &20.587  &26.971  &43.579  &51.278  \\
    Gated    & 61.322 &4.068  &24.947  &30.189  &43.840 &52.497 \\
    \hline
    Saprsemax SAE (Ours) & 61.322 & 21.712 & 33.575 & 37.779 & 47.214 & 59.064   \\
    \bottomrule
  \end{tabular}}
\end{table*}

\begin{figure*}[ht]
    \centering
    \includegraphics[width=0.6\linewidth]{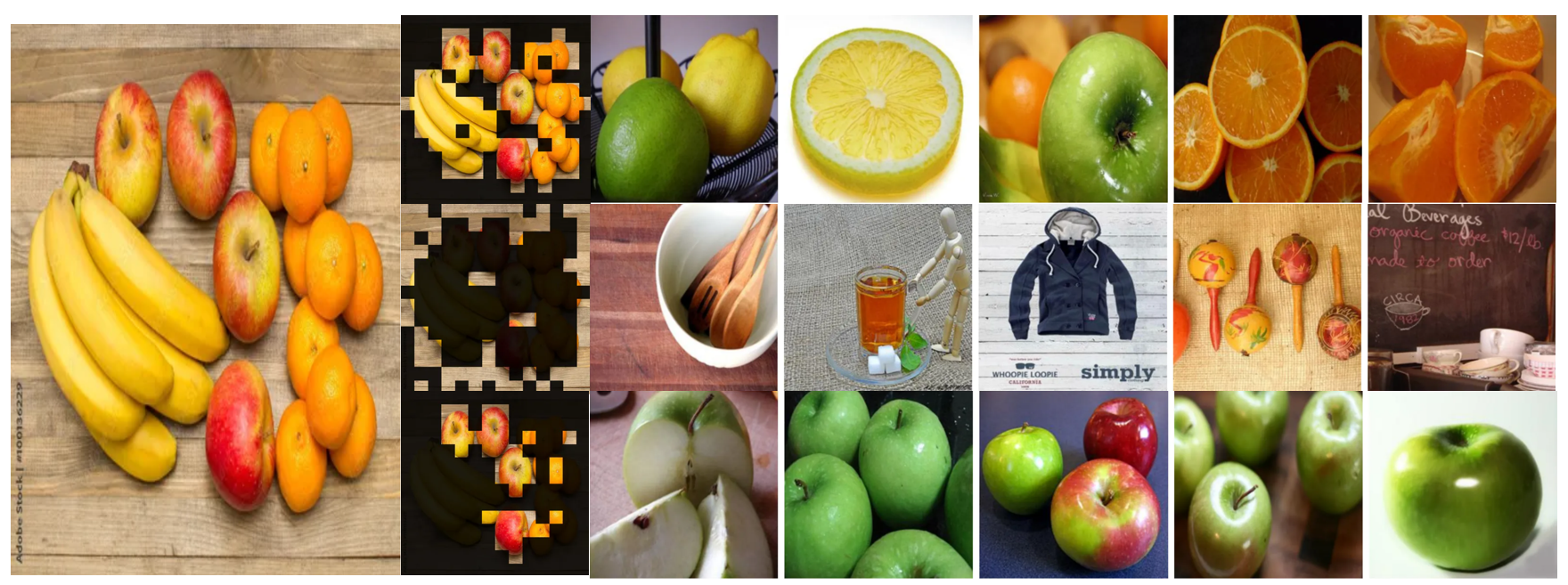}
    \caption{This picture successfully activates three different concepts: fruit, background and apple.}
    \label{fig:fruit}
\end{figure*}

\begin{figure*}[ht]
    \centering
    \includegraphics[width=0.6\linewidth]{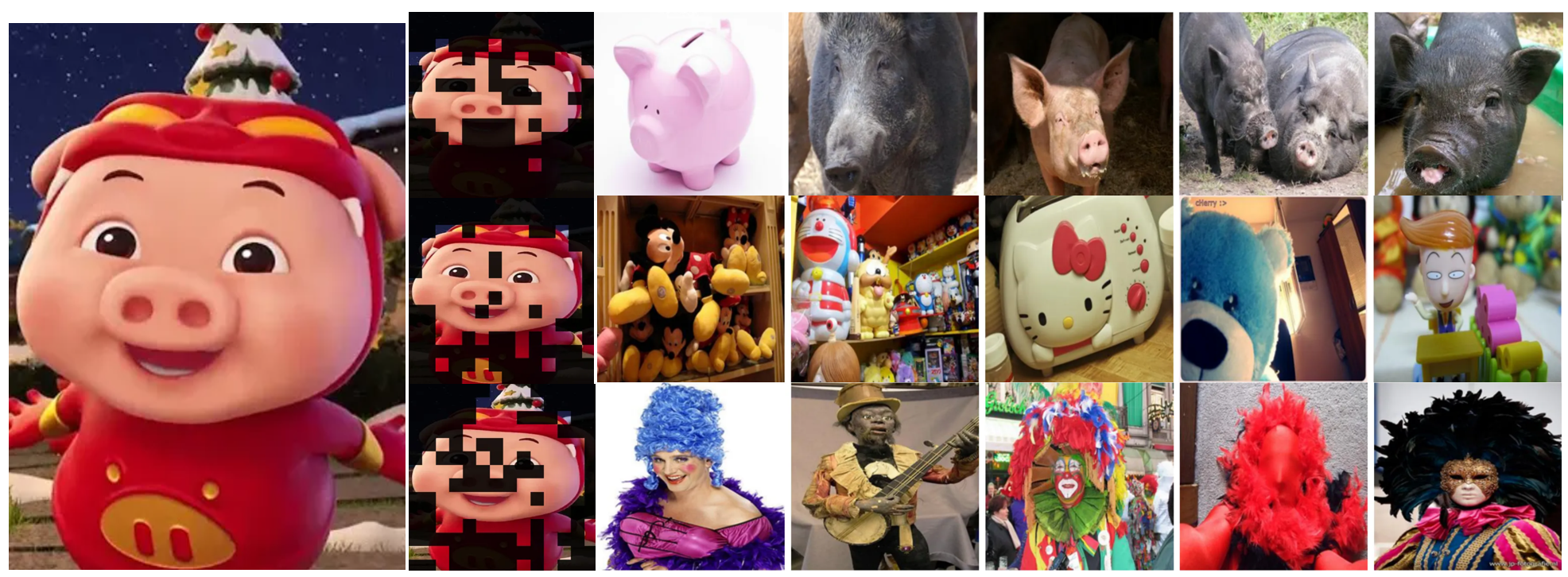}
    \caption{This picture is from the cartoon task image in the animation. It can be seen that it has the characteristics of pigs. Therefore, the features of pig are activated through sae, and other relevant features are activated according to its shape.}
    \label{fig:pig}
\end{figure*}

\begin{figure*}[ht]
    \centering
    \includegraphics[width=0.6\linewidth]{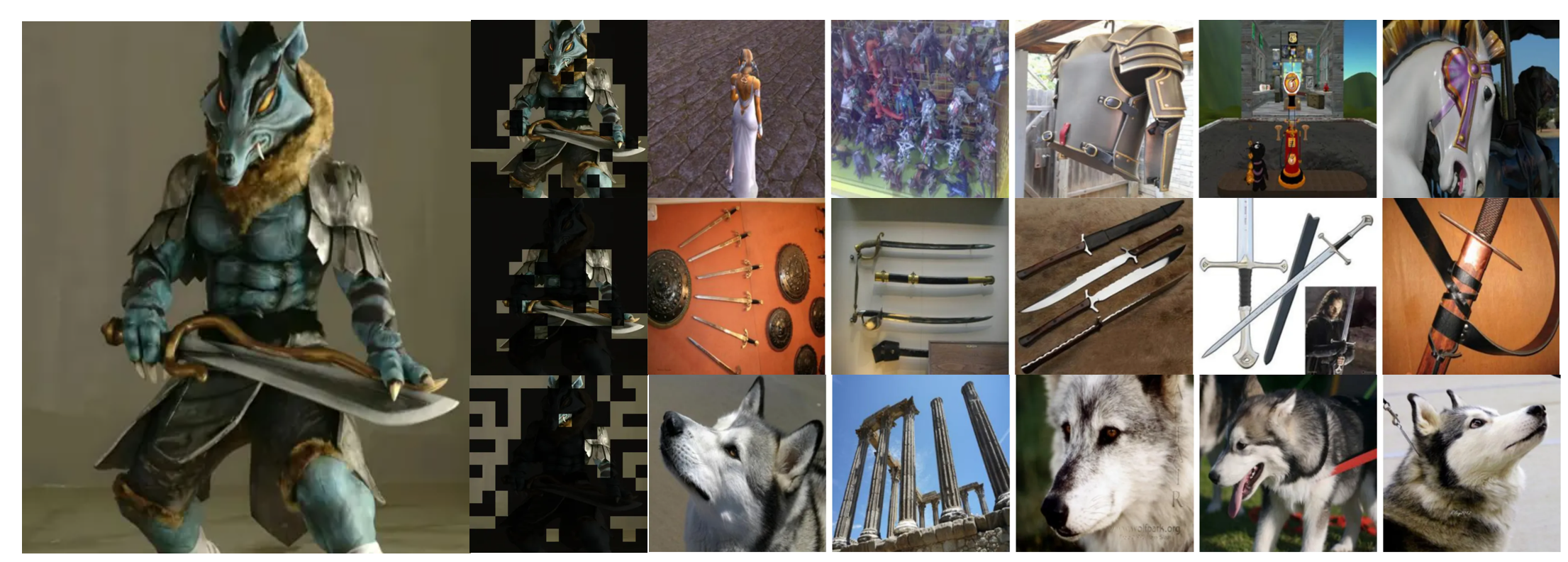}
    \caption{This picture is from the monster in the film and television image. Its wolf's head features activate the wolf's features through sae. In addition, his armor and weapons also activate the corresponding features.}
    \label{fig:wlof}
\end{figure*}

\begin{figure*}[ht]
    \centering
    \includegraphics[width=0.6\linewidth]{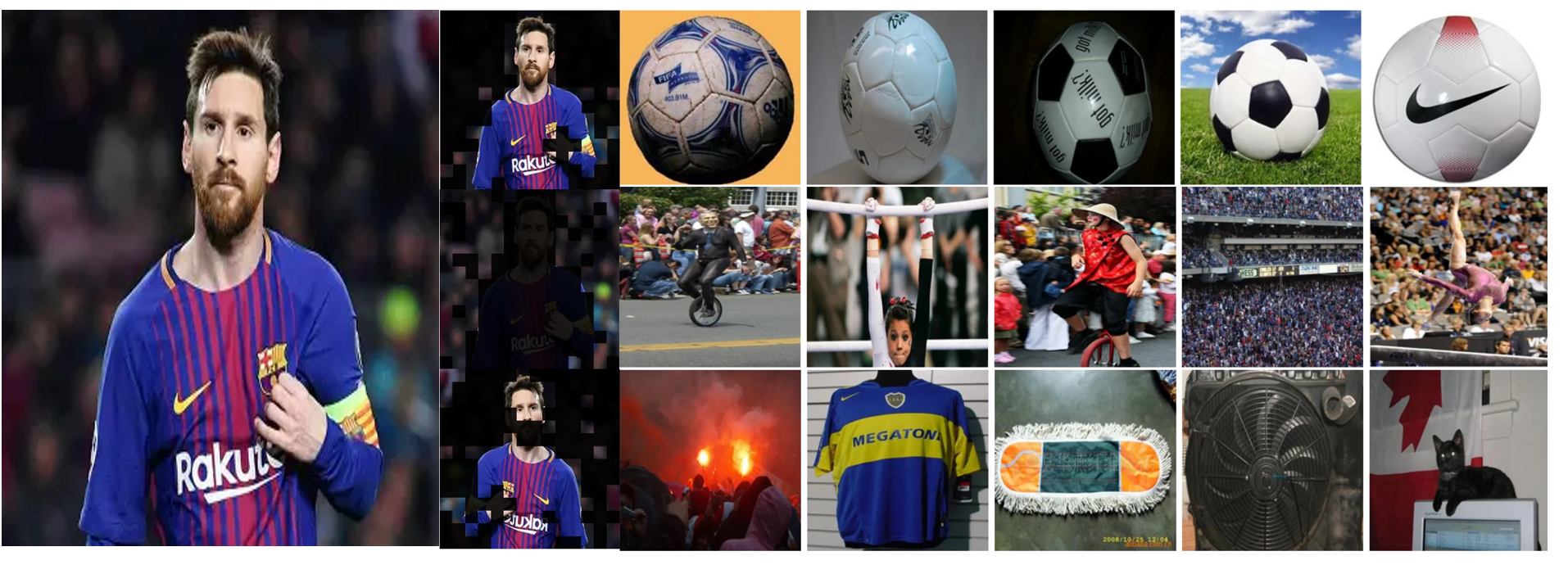}
    \caption{ Messi's picture in this example activated the football. In addition, in this experiment, sae activated the "crowd" feature through the background of the picture,but unfortunately there is a phenomenon of feature absorption in the third feature, where images of short sleeves and similar color structures are considered to be the same concept}
    \label{fig:messi}
\end{figure*}

\begin{figure*}[ht]
    \centering
    \includegraphics[width=0.6\linewidth]{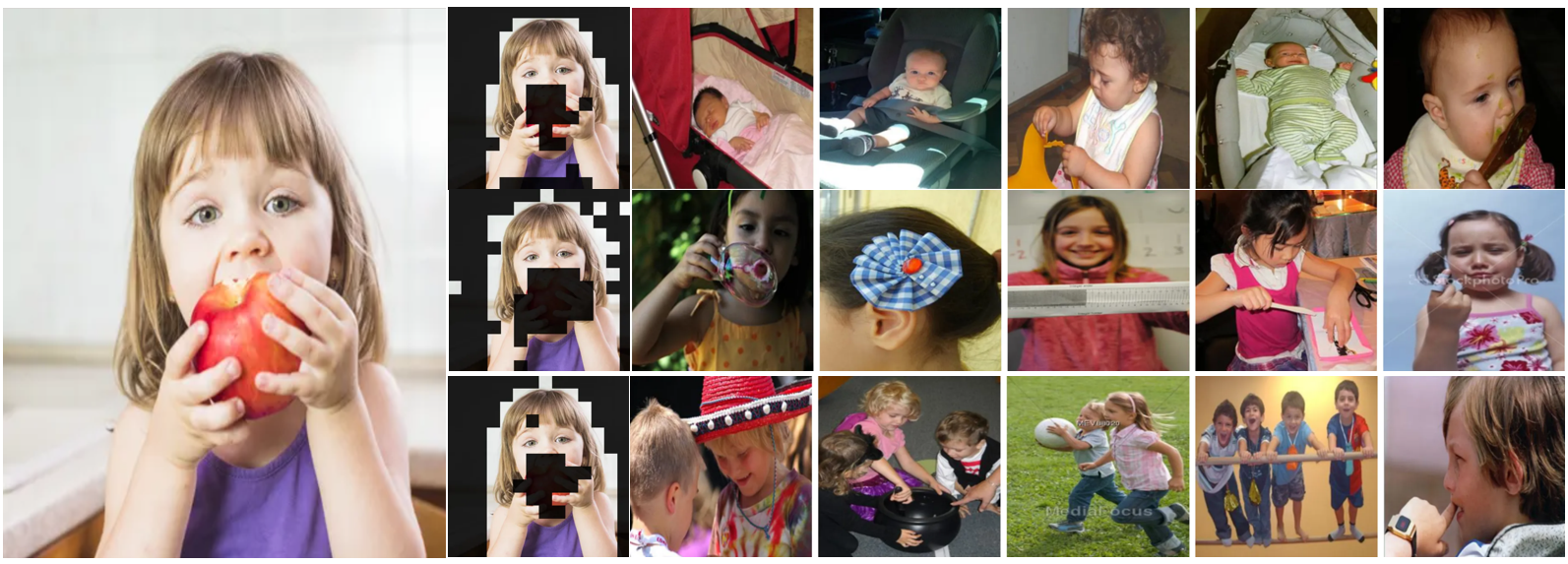}
    \caption{In this sample, pictures activate children and children's eating characteristics through sae.However,the concept of Apple was not recognized and the concept on the right is too similar.}
    \label{fig:girl}
\end{figure*}

\begin{figure*}[ht]
    \centering
    \includegraphics[width=0.6\linewidth]{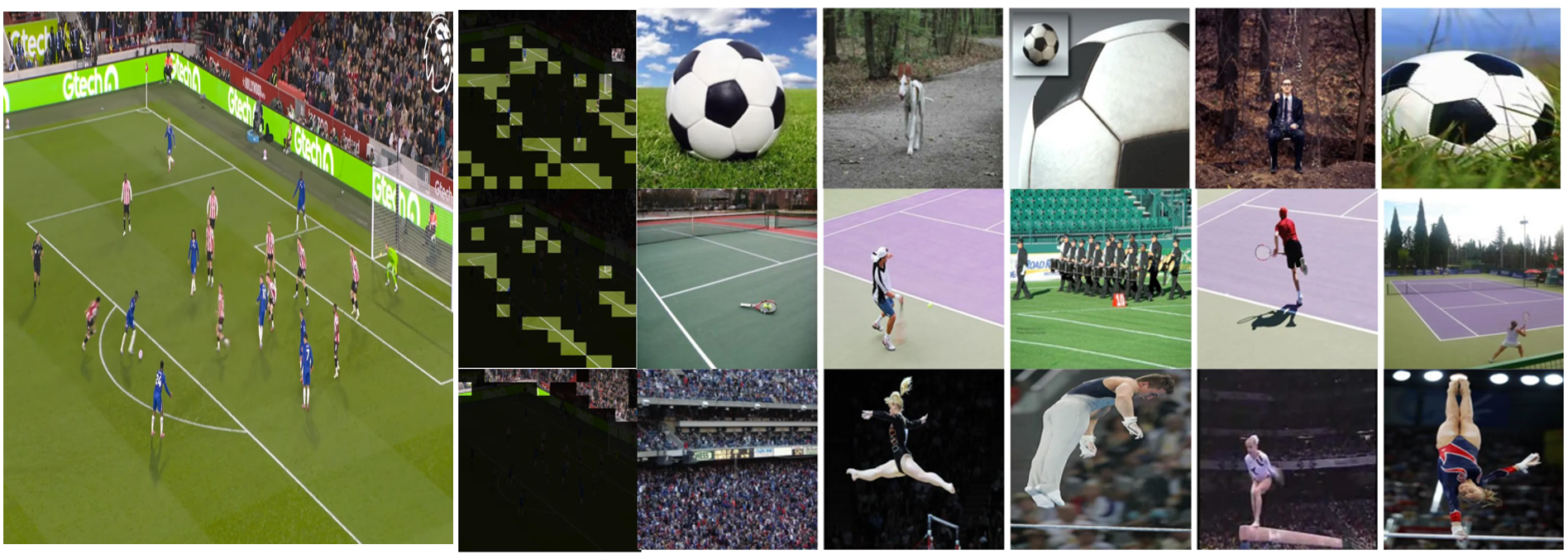}
    \caption{This picture is from a football match. The picture activates the football feature, field , and the crowd .But in the first concept, there were pictures unrelated to football}
    \label{fig:football}
\end{figure*}

\end{document}